\documentclass[conference]{IEEEtran}
\IEEEoverridecommandlockouts
\usepackage[export]{adjustbox}
\usepackage{booktabs}
\usepackage{cite}
\usepackage{amsmath,amssymb,amsfonts}
\usepackage{amsthm}
\usepackage{algorithmic}
\usepackage[hidelinks]{hyperref}
\usepackage{graphicx}
\usepackage[caption=false]{subfig}
\usepackage{textcomp}
\usepackage{tikz}
\usepackage{xcolor}
\def\BibTeX{{\rm B\kern-.05em{\sc i\kern-.025em b}\kern-.08em
    T\kern-.1667em\lower.7ex\hbox{E}\kern-.125emX}}

\usepackage{xkcdcolors}
\usepackage[hidelinks]{hyperref}
\hypersetup{
  colorlinks=true,
  linkcolor=xkcdDarkSkyBlue,
  urlcolor=xkcdDullPink,
  citecolor=xkcdDarkSkyBlue,
}

\newcommand{\bh}{\mathbf{h}}
\newcommand{\bm}{\mathbf{m}}
\newcommand{\bx}{\mathbf{x}}
\newcommand{\bH}{\mathbf{H}}
\newcommand{\bJ}{\mathbf{J}}

\newtheorem{theorem}{Theorem}

\allowdisplaybreaks

\newcommand\copyrighttext{%
  \footnotesize \textcopyright 2025 IEEE. Personal use of this material is permitted. Permission from IEEE must be obtained for all other uses, in any current or future media, including reprinting/republishing this material for advertising or promotional purposes, creating new collective works, for resale or redistribution to servers or lists, or reuse of any copyrighted component of this work in other works. DOI: \href{https://doi.org/10.1109/IJCNN64981.2025.11227859}{10.1109/IJCNN64981.2025.11227859}.}
\newcommand\copyrightnotice{%
\begin{tikzpicture}[remember picture,overlay]
\node[anchor=south,yshift=10pt] at (current page.south) 
  {\fbox{\parbox{\dimexpr\textwidth-\fboxsep-\fboxrule\relax}{\copyrighttext}}};
\end{tikzpicture}%
}
\begin{document}

\title{Residual Reservoir Memory Networks
\thanks{Published in IJCNN2025 - Special Session on Reservoir computing in the deep learning era: theory, models, applications, and hardware implementations. This work has been supported by NEURONE, a project funded by the European Union - Next Generation EU, M4C1 CUP I53D23003600006, under program PRIN 2022 (prj. code 20229JRTZA), and by EU-EIC EMERGE (Grant No. 101070918).}
}

\author{
\IEEEauthorblockN{1\textsuperscript{st} Matteo Pinna}
\IEEEauthorblockA{\textit{Department of Computer Science} \\
\textit{University of Pisa} \\
Pisa, Italy \\
matteo.pinna@di.unipi.it}
\and
\IEEEauthorblockN{2\textsuperscript{nd} Andrea Ceni}
\IEEEauthorblockA{\textit{Department of Computer Science} \\
\textit{University of Pisa} \\
Pisa, Italy \\
andrea.ceni@di.unipi.it}
\and
\IEEEauthorblockN{3\textsuperscript{rd} Claudio Gallicchio}
\IEEEauthorblockA{\textit{Department of Computer Science} \\
\textit{University of Pisa} \\
Pisa, Italy \\
claudio.gallicchio@unipi.it}
}

\maketitle
\copyrightnotice
\begin{abstract}
We introduce a novel class of untrained Recurrent Neural Networks (RNNs) within the Reservoir Computing (RC) paradigm, called Residual Reservoir Memory Networks (ResRMNs). ResRMN combines a linear memory reservoir with a non-linear reservoir, where the latter is based on residual orthogonal connections along the temporal dimension for enhanced long-term propagation of the input. The resulting reservoir state dynamics are studied through the lens of linear stability analysis, and we investigate diverse configurations for the temporal residual connections. The proposed approach is empirically assessed on time-series and pixel-level 1-D classification tasks. Our experimental results highlight the advantages of the proposed approach over other conventional RC models\footnote{Code is available at \href{https://github.com/NennoMP/residualrmn}{github.com/NennoMP/residualrmn}.}.
\end{abstract}


\section{Introduction}
The Reservoir Computing (RC) paradigm\cite{nakajima2021reservoir, lukovsevivcius2009reservoir} is a unique approach for the design of untrained Recurrent Neural Networks (RNNs), popular for its computational efficiency. Training is limited to a simple linear readout and can be performed by means of light-weight closed-form solutions (e.g., ridge regression). The RNN component, commonly called the \emph{reservoir}, is left untrained after its random initialization. The theoretical framework on which RC lays its foundations makes these systems particularly suitable for implementation through physical systems, leading to research efforts in neuromorphic computing\cite{tanaka2019recent, yan2024emerging} and more specifically in RC-based nanowire networks\cite{milano2022materia, pistolesi2024nanowire}. Another prominent research direction has focused on developing increasingly advanced RC-based architectures \cite{gallicchio2017deepesn, gallicchio2024euler, ceni2024residual} to tackle the persistent challenge of learning long-term dependencies in sequential data, a common limitation affecting recurrent models.

In this paper, we introduce \emph{Residual Reservoir Memory Networks} (ResRMNs). ResRMNs are characterized by a hierarchical and modular structure, combining a linear reservoir and a non-linear reservoir. Conceptually, the linear reservoir is designed to enhance memorization capabilities, while the non-linear reservoir, whose implementation in this work is based on temporal residual connections, is designed to capture complex, non-linear dependencies.
We study different configurations for the temporal residual connections, including orthogonal matrices with either random or fixed structure. Additionally, we provide a characterization of the spectrum of ResRMN's Jacobian in Th. \ref{thm:spectrum}, and analyze stability properties through the lens of linear stability analysis.
Through extensive empirical evaluation on time-series classification and pixel-level 1-D classification tasks, we demonstrate the advantages of the proposed approach over other RC-based models.

The rest of this paper is organized as follows. Section \ref{sec:background} provides an overview of relevant related works within the RC paradigm. The proposed approach, ResRMN, is introduced in Section \ref{sec:resrmn}. Section \ref{sec:stability} is dedicated to a linear stability analysis of the dynamics of ResRMN. Section \ref{sec:experiments} is dedicated to the experiments. Finally, Section \ref{sec:conclusions} concludes the paper and suggests future research directions.

\section{Reservoir Computing}
\label{sec:background}
Echo State Networks (ESNs) \cite{jaeger2004harnessing} are probably one of the most popular types of RC models. They are based on an untrained reservoir and a trainable readout layer, and the reservoir is generally initialized with respect to some stability constraints. One of the most basic forms of ESNs is the \emph{Leaky Echo State Network} (LeakyESN)\cite{jaeger2007leakyesn}. The state update function of a LeakyESN reads:
\begin{equation}
    \mathbf{h}(t) = (1 - \tau) \mathbf{h}(t - 1) + \tau \tanh \big( \mathbf{W}_{h}\,\mathbf{h}(t - 1) + \mathbf{W}_{x}\,\mathbf{x}(t) + \mathbf{b} \big),
    \label{eq:leakyesn}
\end{equation}
where $\mathbf{h}(t) \in \mathbb{R}^{N_h}$ and $\mathbf{x}(t) \in \mathbb{R}^{N_x}$ denote the state and the external input at time step $t$, respectively. Additionally, $\mathbf{W}_{h} \in \mathbb{R}^{N_{h} \times N_{h}}$ is the recurrent weight matrix, $\mathbf{W}_{x} \in \mathbb{R}^{N_{h} \times N_{x}}$ is the input weight matrix, $\mathbf{b} \in \mathbb{R}^{N_{h}}$ is the bias vector, and $\tau \in (0, 1]$ is the leaky rate. The weight matrices and bias vector are randomly initialized and left untrained. Additionally, the recurrent weight matrix $\mathbf{W}_{h}$ is rescaled to match a specific spectral radius, denoted as $\rho$, which is a crucial hyperparameter that determines whether the dynamics of the resulting model are stable or unstable. The readout layer can be expressed as $\mathbf{y}(t) = \mathbf{W}_{o} \mathbf{h}(t)$, where $\mathbf{y}(t) \in \mathbb{R}^{N_o}$ denotes the outputs at time step $t$ and $\mathbf{W}_{o} \in \mathbb{R}^{N_o \times N_h}$ denotes the readout weight matrix. The weights within matrix $\mathbf{W}_{o}$ constitute the only trainable parameters in the entire model.

Other relevant approaches worth mentioning are \emph{Residual Echo State Networks} (ResESNs)\cite{ceni2024residual} and \emph{Reservoir Memory Networks} (RMNs)\cite{gallicchioreservoir}. ResESNs extend the LeakyESN by integrating various configurations of temporal residual connections. Additionally, the leaky rate is replaced by two independent scaling coefficients that weight the linear and non-linear branches. RMNs consist of a linear reservoir and a non-linear LeakyESN, where the former influences the dynamics of the latter. In particular, the linear reservoir is constrained to a ring topology, which has been shown to provide noticeable advantages compared to random initializations \cite{rodan2010minimum, verzelli2020input, tino2020dynamical}.

\section{Residual Reservoir Memory Network}
\label{sec:resrmn}
The structure of ResRMN involves a dual-reservoir approach, where a linear memory reservoir and non-linear reservoir are combined. The non-linear module is implemented as a ResESN, a particular type of untrained RNN based on temporal residual connections. The linear module is conceptually optimized for long-term memorization, while the non-linear module is better suited for handling complex dependencies. Fig. \ref{fig:resrmn_architecture} graphically illustrates the proposed architecture. The memory reservoir is a linear processing module driven solely by the external input, while the ResESN module receives in input both the external input and the output states of the memory reservoir and it non-linearly integrates their sums with its hidden state.
The ResESN's output is then fed to a readout layer, which is the only component being trained in the model.

\begin{figure}[t]
    \centering
    \includegraphics[width=\columnwidth]{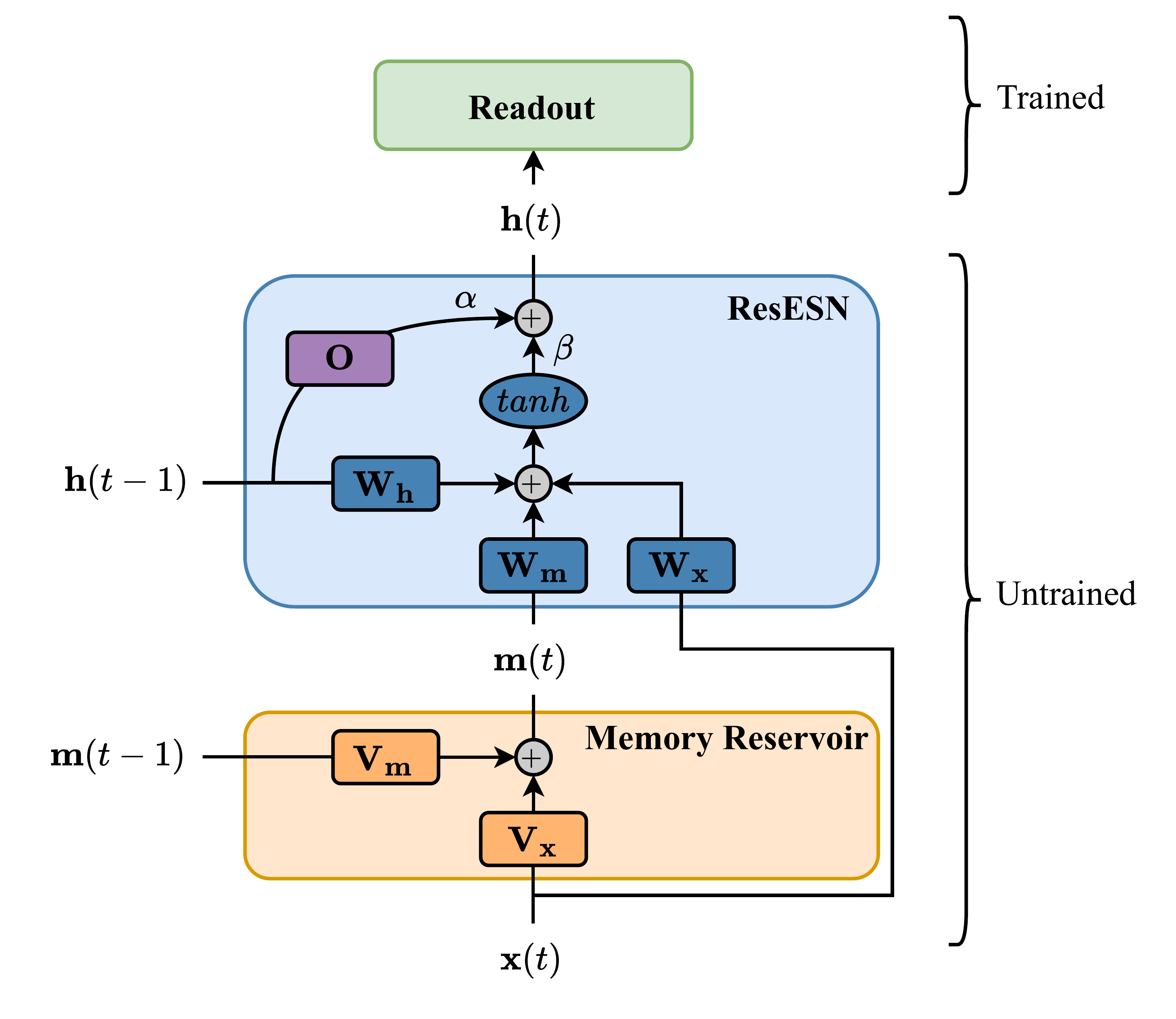}
    \caption{The architecture of a Residual Reservoir Memory Network (ResRMN), assuming the hyperbolic tangent $tanh$ as the activation function and an orthogonal matrix $\mathbf{O}$ in the residual branch of the Residual Echo State Network (ResESN). The architecture consists of two untrained components: (1) a linear memory reservoir driven by the external input $\mathbf{x}$, and (2) a non-linear residual reservoir driven by both the external input $\mathbf{x}$ and the output of the memory reservoir $\mathbf{m}$. The final output is fed to a readout layer, which is the only trainable component.}
    \label{fig:resrmn_architecture}
\end{figure}

In the context of the memory reservoir, we denote $N_x$ and $N_m$ as the number of input neurons and the hidden size, respectively. The linear memory reservoir state update function can be expressed as:
\begin{equation}
\mathbf{m}(t) = \mathbf{V}_{m}\mathbf{m}(t - 1) + \mathbf{V}_{x}\mathbf{x}(t),
\label{eq:linear_reservoir}
\end{equation}
where $\mathbf{V}_{m} \in \mathbb{R}^{N_m \times N_m}$ and $\mathbf{V}_{x} \in \mathbb{R}^{N_m \times N_x}$ are the recurrent weight matrix and the input weight matrix, respectively.

In the context of the ResESN module, we denote $N_h$ as the number of recurrent neurons in the reservoir. The non-linear state update function of ResESN reads:
\begin{align}
    \mathbf{h}(t) = \alpha\,\mathbf{O}\,\mathbf{h}(t - 1) + \beta\,\tanh \big(& \mathbf{W}_{h}\,\mathbf{h}(t - 1) + \mathbf{W}_{m}\,\mathbf{m}(t) \notag \\ 
    &+ \mathbf{W}_{x}\,\mathbf{x}(t) + \mathbf{b}_{h} \big),
    \label{eq:nonlinear_reservoir}
\end{align}
where $\mathbf{O} \in \mathbb{R}^{N_h \times N_h}$ is an orthogonal matrix, $\mathbf{W}_{h} \in \mathbb{R}^{N_h \times N_h}$ is the recurrent weight matrix, $\mathbf{W}_{m} \in \mathbb{R}^{N_h \times N_m}$ is the memory weight matrix, $\mathbf{W}_{x} \in \mathbb{R}^{N_h \times N_x}$ is the input weight matrix, and $\mathbf{b}_{h} \in \mathbb{R}^{N_{h}}$ is the bias vector. Additionally, $\alpha \in [0, 1]$ and $\beta \in (0, 1]$ are two independent scaling coefficients.

Similarly to \cite{ceni2024residual}, we will consider different ResRMN configurations based on the specific structure employed for the orthogonal matrix $\mathbf{O}$. More specifically, we will consider a random variant ResRMN$_{\mathrm{R}}$ where matrix $\mathbf{O}$ is a random orthogonal matrix obtained via QR decomposition of a random matrix $\mathbf{M}^{N_h \times N_h}$ with i.i.d entries in $[-1, 1)$, and two structured configurations ResRMN$_{\mathrm{C}}$ and ResRMN$_{\mathrm{I}}$ where the orthogonal matrix is implemented as a cyclic orthogonal matrix $\mathbf{C}$, whose structure is highlighted in (\ref{eq:cyclic_ortho_matrix}), and as the identity matrix $\mathbf{I}$, respectively. We choose these three configurations to cover a broad spectrum of dynamical behaviors \cite{jaeger2001memory}. In particular, the identity matrix is expected to provide relatively low memory capacity, while the other orthogonal configurations are expected to maximize it \cite{ceni2024edge}.

\begin{equation}
    \mathbf{C} = 
    \begin{bmatrix}
        0 & 0 & \cdots & 0 & 1 \\
        1 & 0 & \cdots & 0 & 0 \\
        0 & 1 & \cdots & 0 & 0 \\
        \vdots & \vdots & \ddots & \vdots & \vdots \\
        0 & 0 & \cdots & 1 & 0
    \end{bmatrix}
    \label{eq:cyclic_ortho_matrix}
\end{equation}
Note that the configuration with the identity matrix, ResRMN$_{\mathrm{I}}$, becomes equivalent to an RMN when $\alpha = 1 - \beta$. Furthermore, a ResRMN reduces to a simple ResESN when either the number of memory neurons $N_{m}$ equals $0$ or the memory weight matrix $W_{m}$ is a zero matrix, as the memory module would not contribute to the resulting dynamics.

\begin{figure}[t]
    \includegraphics[scale=1.0, left]{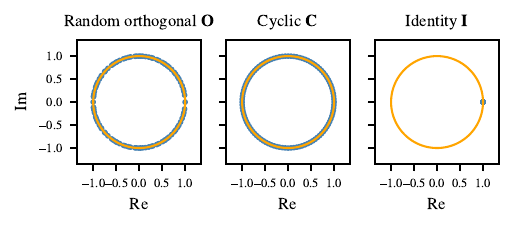}
    \caption{Eigenvalues of the different orthogonal matrices considered in the non-linear module, assuming a hidden size of $N_{h} = 100$. In orange the unitary circle.}
    \label{fig:ortho_eigenvalues}
\end{figure}

Following the approach employed in \cite{gallicchioreservoir}, the memory reservoir's recurrent weight matrix $\mathbf{V}_{m}$ adopts the same cyclic structure highlighted in (\ref{eq:cyclic_ortho_matrix}). Fig. \ref{fig:ortho_eigenvalues} visualizes, with respect to the unitary circle, the eigenvalue distribution of the three orthogonal matrices considered. We observe that the random orthogonal matrix and the cyclic orthogonal matrix exhibit similar distribution spread around the unitary circle, while the identity matrix's distribution is characterized by zero imaginary part and strictly positive real part. These differences, in addition to their different structure and sparsity, may effectively introduce different architectural biases in the model's dynamics. Therefore, the different configurations are worth investigating.

\section{Linear Stability Analysis}
\label{sec:stability}
In this section, we perform a linear stability analysis of ResRMN. Linear stability analysis is a fundamental technique used to evaluate the stability of equilibrium states in dynamical systems. By analyzing small perturbations around a reference trajectory, it is possible to gain critical insights into whether the system will stay around the reference trajectory, or exhibit instability.
To achieve this, we linearize the governing equations of the system by computing the Jacobian matrix, which represents the first-order partial derivatives of the system's vector field with respect to its state variables. The Jacobian, denoted as $\mathbf{J}$, captures how infinitesimal perturbations propagate in the system.
The stability of the system is then determined by the eigenvalues of the Jacobian matrix. Specifically, we compute the spectral radius, which is the largest absolute value of the eigenvalues of $\mathbf{J}$. If the spectral radius is less than $1$, all perturbations decay over time, indicating that the system is locally stable around the reference trajectory. Conversely, if the spectral radius is greater than or equal to $1$, at least one perturbation grows over time, leading to instability.
Thus, by evaluating the eigenvalues of the linearized system, we can determine whether small disturbances will decay (signifying stability) or amplify (indicating instability). This analysis provides valuable insights into the robustness of the equilibrium state and helps assess the resilience of the system against small fluctuations.
We denote $ \bH(t)=  \begin{pmatrix} \bm(t) \\ \bh(t) \end{pmatrix}$, then the global state-update function can be described as:
\begin{equation}
    \label{eq:state_update}
    \bH(t) = \begin{pmatrix} \mathbf{F}_{\bx(t)}^{(1)}(\bm(t-1)) \\ \mathbf{F}_{\bx(t)}^{(2)}\Bigl(\bh(t-1),\mathbf{F}_{\bx(t)}^{(1)}\bigl(\bm(t-1)\bigl)\Bigl) \end{pmatrix},
\end{equation}
where $ \mathbf{F}_{\bx(t)}^{(1)} $ represents the linear memory reservoir state update function defined in (\ref{eq:linear_reservoir}), and $ \mathbf{F}_{\bx(t)}^{(2)} $ represents the non-linear state update function of the ResESN module defined in (\ref{eq:nonlinear_reservoir}).
We aim to compute the Jacobian of ResRMN, which reads:
\begin{align}
    \label{eq:jacobian}
    \bJ_{\text{ResRMN}} &= \left[\begin{array}{cc}
    \frac{\partial \bm(t)}{\partial  \bm(t-1) } & \frac{\partial  \bm(t)}{\partial  \bh(t-1) } \\
    \frac{\partial \bh(t)}{\partial \bm(t-1) } &  \frac{\partial \bh(t)}{\partial \bh(t-1) }
  \end{array}\right] \notag \\[0.1cm]
  &= \left[\begin{array}{cc}
    \frac{\partial \mathbf{F}_{\bx(t)}^{(1)}}{\partial  \bm(t-1) } &  \mathbf{0} \\
    \frac{\partial \mathbf{F}_{\bx(t)}^{(2)}}{\partial \bm(t-1) } &  \frac{\partial  \mathbf{F}_{\bx(t)}^{(2)}}{\partial \bh(t-1) }
  \end{array}\right] \notag \\[0.1cm]
  &= \left[\begin{array}{cc}
     \mathbf{V}_{m} &  \mathbf{0} \\
     \beta\mathbf{D}_t \mathbf{W}_{m} \mathbf{V}_{m} &   \alpha\mathbf{O} + \beta  \mathbf{D}_t  \mathbf{W}_{h}
  \end{array}\right],
\end{align}
where $\mathbf{D}_t = \text{diag}\Bigl(1-\tanh^2\bigl( \mathbf{W}_{h}\,\mathbf{h}(t - 1) + \mathbf{W}_{m}\,\mathbf{m}(t) + \mathbf{W}_{x}\,\mathbf{x}(t) + \mathbf{b}_{h} \bigl)\Bigl)$.

Due to the hierarchical structure of ResRMN, its Jacobian is a lower triangular block matrix. This leads to the following implications for the eigenvalues and spectral radius of ResRMN.
\begin{theorem}
    \label{thm:spectrum}
The set of eigenvalues of ResRMN is the union of the set of $N_m$ eigenvalues of the linear memory reservoir described by (\ref{eq:linear_reservoir}) and the set of $N_h$ eigenvalues of the non-linear module described by (\ref{eq:nonlinear_reservoir}). In particular, the spectral radius of ResRMN, denoted $\rho_{ResRMN}$, is the maximum spectral radius among the two modules. More formally:
\begin{equation}
    \rho_{\text{ResRMN}} = \max \big( \rho(\mathbf{V}_{m}), \, \rho( \alpha\mathbf{O} + \beta  \mathbf{D}_t  \mathbf{W}_{h}) \big).
\end{equation}
\end{theorem}
\begin{proof}
For triangular matrices $\mathbf{M} = \left[\begin{array}{cc}
    \mathbf{A}_{0,0} & \mathbf{0} \\
    \mathbf{A}_{1, 0} &  \mathbf{A}_{1, 1}
  \end{array}\right]$ the following relation holds:
$$
\det(\mathbf{M} - \lambda \mathbf{I}) = \det(\mathbf{A}_{0, 0} - \lambda \mathbf{I})\det(\mathbf{A}_{1, 1} - \lambda \mathbf{I}).
$$
It follows that the set of eigenvalues of $\mathbf{M}$ is the union of the sets of eigenvalues of the matrices along the diagonal of $\mathbf{M}$, i.e. $\mathbf{A}_{0, 0}$ and $\mathbf{A}_{1, 1}$. Therefore, it also follows that $\rho(M) = \max ( \rho(\mathbf{A}_{0, 0}), \, \rho(\mathbf{A}_{1, 1}) )$.
\end{proof}

As usual in RC, a necessary condition for stability is imposed, specifically that the origin is a stable equilibrium point for the case of zero input and zero bias \cite{jaeger2001echo}. 
In this regard, Theorem \ref{thm:spectrum} allows us to derive a necessary condition for stability of the ResRMN model.
Namely, a necessary condition for the ResRMN model to be stable is that both its linear and non-linear modules satisfy the necessary condition of stability at the origin with zero input and zero bias, 
i.e. $\rho(\mathbf{V}_{m})\leq 1$ and 
$\rho( \alpha\mathbf{O} + \beta  \mathbf{W}_{h}) \leq 1$.\footnote{For the case of zero input and zero bias, the matrix $\mathbf{D}_t$ becomes the identity matrix.}
This follows directly from the hierarchical and modular nature of ResRMN, where both of its modules contribute to the overall dynamical behavior of the system, and its spectral radius being expressed in terms of the spectral radii of its modules. This result underscores the importance of carefully designing both the linear and non-linear components to ensure overall stability. 
In Section \ref{sec:experiments}, we leverage these theoretical insights to guide the exploration of ResRMN's hyperparameters.

Note that, in our implementation, the recurrent weight matrix of the linear module is always initialized to a cyclic orthogonal structure, and has a spectral radius of $1$. Therefore, the resulting model is always characterized by a spectral radius of at least $1$, possibly larger depending on the tuning of the non-linear module. This ensures that ResRMN consistently operates at the edge of stability, a computational regime particularly well-suited for time-series classification, where relevant information may be distributed throughout the entire sequence \cite{ceni2024edge,bertschinger2004real}.

\begin{figure}[t]  
    \includegraphics[scale=1.0, left]{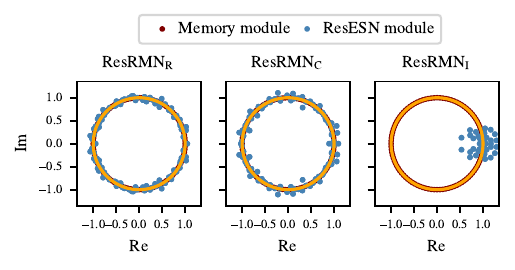}
    \caption{Eigenvalues of the Jacobian for different ResRMN configurations. The dynamics are driven by a random input vector and a random state, both uniformly distributed in $(-1, 1)$. We assume hidden sizes $N_{m}, N_{h} = 100$, a spectral radius $\rho = 1$, all input weight matrices with scaling of $1$ (i.e., $\omega_{x}, \omega_{x_{m}}, \omega_{m} = 1$), zero bias $\omega_{b} = 0$, and scaling coefficients $\alpha, \beta = 1$. In red the $N_{m}$ eigenvalues of the linear memory module, in blue the $N_{h}$ eigenvalues of the ResESN module. In orange the unitary circle.}
    \label{fig:jacobian_eigenvalues}
\end{figure}

\subsection{Eigenspectrum analysis}
The different ResRMN configurations (random orthogonal, cyclic orthogonal, and identity) are intentionally selected so that we may have similar or dissimilar eigenspectra between the linear and non-linear modules. Specifically, similar eigenspectra are expected when the ResESN module employs either the cyclic orthogonal matrix, which matches the structure of the memory cell's recurrent weight matrix $\mathbf{V}_{m}$, or a random orthogonal matrix, which exhibits a very similar eigenvalues' distribution to that of the cyclic orthogonal matrix (see Fig. \ref{fig:ortho_eigenvalues}). Conversely, when the ResESN module employs the identity matrix, the modules should exhibit less redundancy in terms of their eigenspectra due to the dissimilar eigenspectrum that characterizes the identity matrix. Therefore, conceptually, ResRMN$_{\mathrm{I}}$ may capture a broader diversity of representations compared to the other configurations. In Fig. \ref{fig:jacobian_eigenvalues} we visualize the eigenspectrum of the Jacobian defined in (\ref{eq:jacobian}) for each ResRMN configuration, differentiating between the $N_{m}$ eigenvalues of the memory module and the $N_{h}$ eigenvalues of the non-linear module. Trivially, the set of $N_{m}$ eigenvalues of the linear module are spread around the unitary circle, due to its linear dynamics and its recurrent weight matrix being initialized to a cyclic orthogonal structure. Conversely, the distribution of the $N_{h}$ eigenvalues of the non-linear module varies depending on which orthogonal matrix is employed in the temporal residual connections. The random orthogonal and the cyclic orthogonal configurations are characterized by eigenvalues that spread around a $\beta \lVert \mathbf{W}_{h} \rVert$-neighborhood of the circle of radius $\alpha$. The identity configuration is characterized by a more skewed eigenvalues' distribution, where the eigenvalues fall in a $\beta \lVert \mathbf{W}_{h} \rVert$-neighborhood of $\alpha$. 

In future work, this analysis could be complemented by exploring the eigenspectra in polar form, investigating the differences between each configuration from the perspective of eigenvalues' magnitudes and angles, rather than their real and imaginary parts.

\section{Experiments}
\label{sec:experiments}
In this section we validate the proposed approach on time-series classification and pixel-level 1-D classification tasks. More specifically, we consider the following time-series classification datasets from the UEA \& UCR repository\footnote{\href{https://www.timeseriesclassification.com}{https://www.timeseriesclassification.com}}\cite{dau2019ucr}: Adiac, Beef, Blink, Car, DuckDuckGeese (DDG), FordA, FordB, HandMovementDirection (HMD), Libras, Mallat, OSULeaf and Wine. Additionally, we consider the permuted sequential MNIST (psMNIST) pixel-level 1-D classification task, where the traditional MNIST dataset \cite{lecun1998mnist} is flattened to a vector of pixels, so that it can be processed as a sequence, and then randomly permuted. Table \ref{tab:datasets} provides a summary of each dataset.

\begin{table}[t]
\caption{Summary of the classification datasets.}
\begin{center}
\begin{tabular}{llllll}
\toprule
Dataset                 & Train size & Test size & Length & \# Features & \# Classes \\
\midrule
\textbf{Adiac}          & 390       & 391     & 176     & 1     & 37\\
\textbf{Beef}           & 30        & 30      & 470     & 1     & 5 \\
\textbf{Blink}          & 500       & 450     & 510     & 4     & 2 \\
\textbf{Car}            & 60        & 60      & 577     & 1     & 4 \\   
\textbf{DDG}            & 60        & 40      & 270     & 1345  & 5 \\
\textbf{FordA}          & 3601      & 1320    & 500     & 1     & 2 \\
\textbf{FordB}          & 3636      & 810     & 500     & 1     & 2 \\
\textbf{HMD}            & 160       & 74      & 400     & 10    & 4 \\
\textbf{Libras}         & 180       & 180     & 45      & 2     & 15 \\
\textbf{Mallat}         & 55        & 2345    & 1024    & 1     & 8 \\
\textbf{psMNIST}        & 60000     & 10000   & 784     & 1     & 10 \\
\textbf{OSULeaf}        & 200       & 242     & 427     & 1     & 6 \\     
\textbf{Wine}           & 57        & 54      & 234     & 1     & 2 \\     
\bottomrule
\end{tabular}
\label{tab:datasets}
\end{center}
\end{table}  

\begin{table}[t]
\caption{Model selection hyperparameters.}
\begin{center}
\begin{tabular}{ll}
\toprule
Hyperparameters & Values \\
\midrule
$\omega_{x_{m}}$ & $[0.01, 0.1, 1]$ \\
$\omega_x$ & $[0.01, 0.1, 1]$ \\
$\omega_{m}$ & $[0.01, 0.1, 1]$ \\
$\omega_b$ & $[0, 0.01, 0.1, 1]$ \\
$\rho$ & $[0.9, 1, 1.1]$ \\
\midrule
\textbf{(Leaky variants)} & \\
$\tau$ & $[0.01, 0.1, 0.5, 0.9, 0.99, 1]$ \\
\midrule
\textbf{(Residual variants)} & \\
$\alpha$ & $[0, 0.01, 0.1, 0.5, 0.9, 0.99, 1]$ \\
$\beta$ & $[0.01, 0.1, 0.5, 0.9, 0.99, 1]$ \\
\bottomrule
\end{tabular}
\label{tab:model_selection}
\end{center}
\end{table}

\begin{table*}[htbp]
\caption{Performance achieved on the test set of time-series classification tasks. Reported results represent mean and standard deviation over 10 different random initializations. The best overall result is highlighted in blue.}
\begin{center}
\begin{tabular}{llllll|lll}
\toprule
Classification ($\uparrow$) & LeakyESN & ResESN$_{\mathrm{R}}$ & ResESN$_{\mathrm{C}}$ & ResESN$_{\mathrm{I}}$ & RMN & ResRMN$_{\mathrm{R}}$ & ResRMN$_{\mathrm{C}}$ & ResRMN$_{\mathrm{I}}$ \\
\midrule
\textbf{Adiac} & $56.8_{\pm 0.9}$ & $55.2_{\pm 2.6}$ & $54.8_{\pm 4.9}$ & $59.3_{\pm 0.6}$ & $59.6_{\pm 3.5}$ & $60.5_{\pm 3.6}$ & $57.9_{\pm 2.6}$ & \textcolor{blue}{$60.9_{\pm 2.5}$} \\
\textbf{Beef} & $69.3_{\pm 5.9}$ & $79.0_{\pm 3.7}$ & $73.0_{\pm 3.1}$ & $48.7_{\pm 5.8}$ & \textcolor{blue}{$87.0_{\pm 3.3}$} & \textcolor{blue}{$87.0_{\pm 4.8}$} & $77.7_{\pm 5.6}$ & $81.7_{\pm 2.7}$ \\
\textbf{Blink} & $66.5_{\pm 3.9}$ & $59.0_{\pm 2.6}$ & $58.2_{\pm 3.8}$ & $80.0_{\pm 3.9}$ & $73.0_{\pm 8.2}$ & $68.2_{\pm 8.1}$ & $63.7_{\pm 4.1}$ & \textcolor{blue}{$82.0_{\pm 3.7}$} \\
\textbf{Car} & $72.0_{\pm 2.1}$ & $64.8_{\pm 7.4}$ & $65.0_{\pm 4.2}$ & $78.0_{\pm 1.9}$ & $78.7_{\pm 3.5}$ & $74.8_{\pm 5.5}$ & $72.3_{\pm 5.5}$ & \textcolor{blue}{$79.1_{\pm 4.3}$} \\
\textbf{DDG} & $53.4_{\pm 3.7}$ & $47.6_{\pm 3.9}$ & $39.2_{\pm 4.6}$ & $56.2_{\pm 2.6}$ & $56.0_{\pm 5.1}$ & $45.8_{\pm 4.2}$ & $44.4_{\pm 3.2}$ & \textcolor{blue}{$58.2_{\pm 3.8}$} \\
\textbf{FordA} & $69.0_{\pm 1.3}$ & $63.8_{\pm 4.2}$ & $65.0_{\pm 1.6}$ & $68.6_{\pm 1.1}$ & $82.8_{\pm 1.5}$ & $64.5_{\pm 1.9}$ & $59.1_{\pm 2.5}$ & \textcolor{blue}{$88.9_{\pm 0.7}$} \\
\textbf{FordB} & $60.9_{\pm 1.0}$ & $56.9_{\pm 0.7}$ & $56.6_{\pm 1.1}$ & $61.1_{\pm 0.9}$ & $68.5_{\pm 1.7}$ & $55.0_{\pm 2.9}$ & $55.5_{\pm 3.2}$ & \textcolor{blue}{$72.6_{\pm 2.0}$} \\
\textbf{HMD} & $24.7_{\pm 3.5}$ & $27.6_{\pm 2.7}$ & $27.0_{\pm 6.1}$ & $28.8_{\pm 3.1}$ & $41.5_{\pm 4.4}$ & $31.9_{\pm 6.3}$ & $29.7_{\pm 3.7}$ & \textcolor{blue}{$42.0_{\pm 6.4}$} \\
\textbf{Libras} & $76.0_{\pm 1.7}$ & $74.9_{\pm 3.4}$ & $76.6_{\pm 1.7}$ & $75.3_{\pm 2.1}$ & $80.6_{\pm 2.1}$ & $73.0_{\pm 3.8}$ & $71.2_{\pm 4.2}$ & \textcolor{blue}{$80.7_{\pm 1.7}$} \\

\textbf{Mallat} & $87.2_{\pm 0.9}$ & $78.1_{\pm 6.8}$ & $86.6_{\pm 2.1}$ & $87.7_{\pm 0.9}$ & \textcolor{blue}{$88.8_{\pm 2.7}$} & $86.3_{\pm 3.8}$ & $86.9_{\pm 2.6}$ & \textcolor{blue}{$88.8_{\pm 2.9}$} \\

\textbf{OSULeaf} & $52.0_{\pm 1.5}$ & $41.1_{\pm 1.5}$ & $47.6_{\pm 2.9}$ & $51.5_{\pm 2.8}$ & \textcolor{blue}{$59.6_{\pm 2.4}$} & $52.8_{\pm 3.3}$ & $51.7_{\pm 1.9}$ & $59.3_{\pm 2.6}$ \\
\textbf{Wine} & $69.3_{\pm 5.9}$ & $80.4_{\pm 6.4}$ & $81.3_{\pm 4.9}$ & $68.5_{\pm 3.3}$ & $81.5_{\pm 2.5}$ & \textcolor{blue}{$86.1_{\pm 4.9}$} & $84.3_{\pm 2.5}$ & $82.2_{\pm 2.1}$ \\
\bottomrule
\end{tabular}
\label{tab:classification_results}
\end{center}
\end{table*}

\begin{figure*}[htbp]
    \centering
    \begin{minipage}[t]{\textwidth}    
        \begin{minipage}[b]{0.5\textwidth}
            \centering
            \captionsetup[subfloat]{width=\linewidth}
            \subfloat[Performance change relative to $N_{m}$]{\includegraphics[scale=1.0, left]{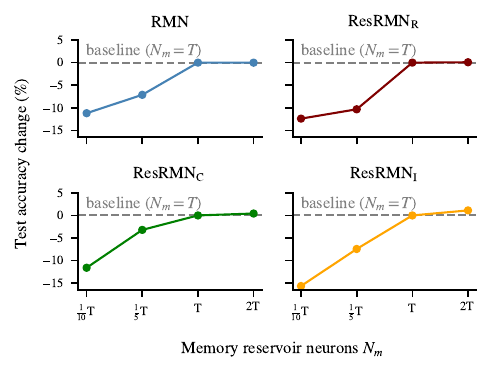}
            \label{fig:performance_change_memunits}}
        \end{minipage}\hfill
        \begin{minipage}[b]{0.5\textwidth}
            \centering
            \captionsetup[subfloat]{width=\linewidth}
            \subfloat[Performance change relative to LeakyESN]{\includegraphics[scale=1.0, center]{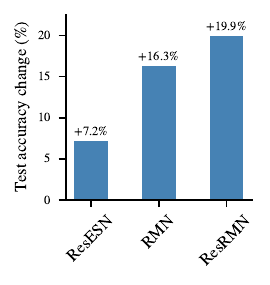}
            \label{fig:performance_change_leakyesn}}
        \end{minipage}\hfill
    \end{minipage}
    \caption{(a) Performance change in time-series classification tasks relative to the number of memory reservoir neurons $N_{m}$, with the base case being $N_{m} = T$. Results are broken down for non-residual and residual reservoir memory networks and averaged across all datasets. (b) Performance change in time-series classification tasks relative to LeakyESN. Results are broken down by model class and averaged across all datasets.}
    \label{fig:performance_changes}
\end{figure*}

\subsection{Model selection}
Model selection is carried out by means of random search. We explore up to $1000$ configurations, with a maximum runtime of $24$h. Results are averaged across $10$ random initializations for time-series classification tasks and $5$ random initializations for pixel-level classification tasks. Note that, for models that have number of possible configurations lower than $1000$ a grid-search is performed. During model selection we use $N_h = 100$ units in the (non-linear) reservoir and $N_m = T$ units in the linear memory cell of RMNs and ResRMNs, where $T$ is the sequence length of the task. The other hyperparameters were explored according to the values highlighted in Table \ref{tab:model_selection}, where $\omega_{x_{m}}$ and $\omega_{x}$ are the scaling coefficients for the input weight matrix of the linear and non-linear module respectively, $\omega_m$ and $\omega_b$ are the scaling coefficients of the memory weight matrix and the bias vector of the non-linear module respectively, and $\rho$ is the desired spectral radius used to rescale the recurrent weight matrix of the non-linear module. For the readout we use the ridge regression implementation from the \emph{Scikit-learn} library, employing a Singular Value Decomposition (SVD) solver, which is fed, in all cases, the state at the last time step. Note that, prior to the readout, the states are transformed by applying \emph{Scikit-learn}'s standard scaler, which removes the mean and scales to unit variance. For each configuration, we explore the following values for the regularization coefficient $\lambda$ of the readout: $[0, 0.01, 0.1, 1, 10, 100]$. The best configuration is chosen based on the classification accuracy achieved on the validation set, which is obtained by applying a $70-30$ stratified split over the training set for time-series classification datasets and a $95-5$ split for the psMNIST dataset. The final results are obtained by re-training models on the original training set.

\begin{figure*}[htbp]
    \centering
    \includegraphics[scale=1.0]{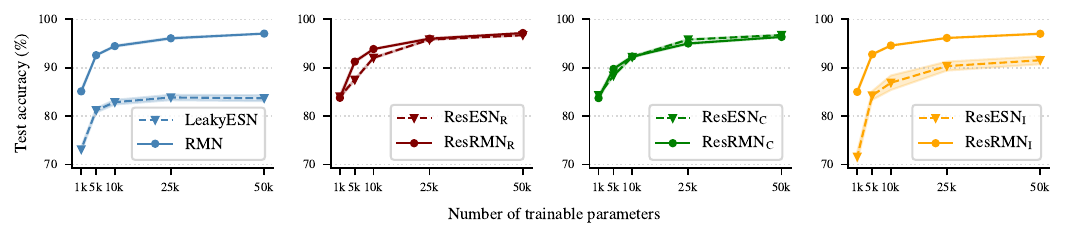}
    \caption{Test results on psMNIST for a varying number of trainable parameters. For RMN and ResRMNs, the number of memory reservoir neurons $N_{m}$ is fixed to $784$. Reported results represent the average, and the corresponding standard deviation, across $5$ random initializations.}
    \label{fig:results_psMNIST}
\end{figure*}

\subsection{Results on time-series classification tasks}
Table \ref{tab:classification_results} presents the results for the considered classification tasks for a hidden size of $N_{h} = 100$. To ensure a fair comparison, all models have the same number of trainable parameters. ResRMN demonstrates superior performance in the majority of datasets, particularly when taking into account the configuration employing the identity matrix in the temporal residual connections. Furthermore, the configuration that employs the cyclic orthogonal matrix in both modules is generally inferior even to the random orthogonal configuration. In some datasets, we observe high degrees of variability (i.e., standard deviations). We argue that this variability may be related to the intrinsic characteristics of each specific dataset. In future work, we plan to investigate in detail how this variability changes as the number of units in the reservoir increases.

Fig. \ref{fig:performance_changes} (right) summarizes the average test accuracy improvements achieved by each class of models (ResESN, RMN, and ResRMN) compared to the LeakyESN baseline across all time-series classification datasets. ResRMN, on average, outperforms the LeakyESN by a $19.9\%$ gain in test accuracy, while also surpassing both ResESN and RMN. Finally, Fig. \ref{fig:performance_changes} (left) analyzes how the number of memory reservoir's neurons $N_m$ affects performance in RMNs and ResRMNs. More specifically, we trained these models for values of $N_{m}$ in $[\frac{1}{10}T, \frac{1}{5}T, T, 2T]$, where $T$ is the sequence length of the dataset, computed performance changes relative to the base case when $N_m = T$, and averaged these changes across all datasets. We observe that $N_{m}$ has a crucial effect on performance, with all modular models decreasing in performance by at least $10\%$ for the lowest number of memory neurons considered (i.e., $N_{m} = \frac{1}{10}T$). In the future, we plan to analyze, with respect to $N_{m}$, the trade-off between performance and training time to investigate the overhead introduced by leveraging two reservoirs rather than one.

\subsection{Results on pixel-level classification tasks}
For psMNIST, we run experiments for an increasing number of trainable parameters from $\approx 1k$ to $\approx 50k$, using the optimal configuration identified during model selection with $N_{h} = 100$. We adjust the hidden size to reach the desired parameter count, while tuning the readout's regularization coefficient via validation set to prevent overfitting. Throughout these experiments, the memory reservoir neuron count $N_{m}$ remains fixed at the MNIST's sequence length $T = 784$. Fig. \ref{fig:results_psMNIST} presents the results of the aforementioned experiment. Results show minimal variation among RMN and the different ResRMN configurations, as we can observe similar performance across all trainable parameter ranges. Notably, we observe that models employing a dual-reservoir approach generally improve upon their single reservoir counterparts, especially those employing the identity matrix, such as RMN and ResRMN$_{\mathrm{I}}$. These results suggest that orthogonal configurations have, in the context of the psMNIST task, strong performance even without a memory module, while configurations using identity matrices fall behind without the additional linear reservoir.

\section{Conclusions}
\label{sec:conclusions}
In this paper, we introduced Residual Reservoir Memory Networks (ResRMNs), a novel class of models within the RC paradigm, aimed at enhancing long-term propagation of the input and performance on time-series classification tasks. The proposed approach is based on the combination of a linear memory reservoir and a non-linear ResESN module, where the latter is a particular class of RC models that take advantage of temporal residual connections. The linear reservoir is conceptually optimized to enhance memorization capabilities, while the ResESN module is optimized for complex, non-linear processing. Additionally, the two have independent hyperparameters, allowing for more flexibility in terms of tuning compared to other RC approaches based on a single reservoir, such as LeakyESN or ResESN. 

Leveraging linear stability analysis tools, we studied the resulting dynamics of the proposed approach. Additionally, to gain insights on how the different configurations affect eigenvalues' distribution and possibly system dynamics, we investigated the eigenspectrum of the resulting Jacobian.

Experimental results on time-series classification tasks demonstrate the superior performance of the proposed approach, which outperforms other RC models in the majority of datasets. In particular, the configuration employing the identity matrix is generally the best performer.

In future work, we plan to experiment with different initialization approaches for the memory cell's recurrent weight matrix and more advanced architectures for the linear memory reservoir. Future work could also analyze the eigenvalues of the proposed model in terms of their polar form (magnitude and angle) rather than their cartesian form (real and imaginary parts). Leveraging this different perspective could provide insights into the differences between the orthogonal configurations and the identity one, particularly in understanding the rotational dynamics and stability properties introduced by each configuration.

\section*{Acknowledgments}
Computational resources were provided by Computing@Unipi, a computing service of the University of Pisa.

\bibliographystyle{IEEEtran}
\bibliography{bibliography}

\end{document}